\documentclass[12pt]{article}
\usepackage{times}
\usepackage{epsf,epsfig}
\usepackage{subfigure}
\usepackage{natbib}
\usepackage{amsmath,amssymb,amsfonts,amsthm}

\renewcommand{\L}{\mathcal L} % likelihood
\renewcommand{\P}{\mathcal P} % also probability
\newcommand{\ignore}[1]{}

\newsavebox{\savepar}

\newenvironment{bigboxit}{\begin{center}\begin{lrbox}{\savepar}
\begin{minipage}[h]{4.5in}
\sffamily
\begin{flushleft}}
{\end{flushleft}\end{minipage}\end{lrbox}\fbox{\usebox{\savepar}}
\end{center}}

\newtheorem{theorem}{Theorem}[section]

\title{Discrete Temporal Models of Social Networks}
\author{Steve Hanneke {\hskip 10mm} Wenjie Fu {\hskip 10mm} Eric P. Xing \\
School of Computer Science\\
Carnegie Mellon University\\
Pittsburgh, PA 15213 USA\\ 
\{shanneke,wenjief,epxing\}@cs.cmu.edu}
\begin{document}

\bibliographystyle{natbib}

\maketitle

%\begin{center}
%{\Large Draft: \today}
%\end{center}

\begin{abstract}
We propose a family of statistical models for social network evolution over time,
which represents an extension of Exponential Random Graph Models (ERGMs).  Many of the
methods for ERGMs are readily adapted for these models, including
maximum likelihood estimation algorithms.  We discuss models of this type and their properties, and
give examples, as well as a demonstration of their use for hypothesis testing and classification. We believe our temporal ERG models represent a useful new framework for modeling time-evolving social networks, and rewiring networks from other domains such as gene regulation circuitry, and communication networks.   
\end{abstract}

\section{Introduction}

% motivate statistical modeling of social networks
The field of social network analysis is concerned with populations of
\emph{actors}, interconnected by a set of \emph{relations} (e.g.,
friendship, communication, etc.).  These relationships can be
concisely described by directed graphs, with one vertex for each actor
and an edge for each relation between a pair of actors.  This network
representation of a population can provide insight into organizational
structures, social behavior patterns, emergence of global structure
from local dynamics, and a variety of other social phenomena.

There has been increasing demand for flexible statistical models of
social networks, for the purposes of scientific exploration and as a
basis for practical analysis and data mining tools.
The subject of modeling a static social network has been investigated
in some depth. For time-invariant networks, represented
as a single directed or undirected graph, a number of flexible
statistical models have been proposed, including the classic
Exponential Random Graph Models (ERGM) and
extensions~\citep{frank:86,wasserman:05,snijders:02,robins:05},
which are descriptive in nature, latent space models that aim
towards clustering and community discovery
\citep{handcock:07}, and mixed-membership block models for
role discovery \citep{airoldi:08}. Of particular relevance to this paper is the ERGM, 
which is particularly flexible in that it can be customized to capture a wide range of 
signature connectivity patterns in the network via user-specified functions 
representing their sufficient statistics.
Specifically, if $N$ is some representation of a social network, and
$\mathcal N$ is the set of all possible networks in this
representation, then the probability distribution function for any
ERGM can be written in the following general form.
\begin{equation*}
\P(N) = \frac{1}{Z({\boldsymbol \theta})}\exp\left\{{\boldsymbol \theta}^\prime {\mathbf u}(N)\right\}.
\end{equation*}
Here, ${\boldsymbol \theta} \in {\mathbb R}^k$, and ${\mathbf u} : {\mathcal N} \rightarrow {\mathbb R}^k$.
$Z({\boldsymbol \theta})$ is a normalization constant, which is typically intractable to compute.
The $\mathbf u$ function represents the sufficient statistics for the model, and, in a graphical modeling
interpretation, can be regarded as a vector of clique potentials.  The representation for $N$ can vary
widely, possibly including multiple relation types, valued or binary relations, symmetric or asymmetric
relations, and actor and relation attributes.  The most widely studied models of this form are for
single-relation social networks, in which case $N$ is generally taken to be the weight matrix $A$ for the
network (sometimes referred to as a \emph{sociomatrix}), where $A_{ij}$ is the strength of directed
relation between the $i^{th}$ actor and $j^{th}$ actor.

Often one is interested in modeling the evolution of a network over multiple sequential observations.
For example, one may wish to model the evolution of coauthorship networks in a specific community from year to year,
trends in the evolution of the World Wide Web,
or a process by which simple local relationship dynamics give rise to global structure.
%One would ideally like a model family that is capable of modeling network
%evolution, while maintaining the flexibility of ERGMs.  One would also
%like such models to build upon ERGMs as much as possible, so that
%existing methods developed for ERGMs over the past two
%decades are readily adapted to apply to the temporal models as
%well.  In the following sections, we propose such a family.
In such dynamic settings, where a
time-series of observations of the network structure is available,
several formalisms have been proposed to model the dynamics of
topological changes of such networks over time. For example, \citet{snijders:06} has proposed a continuous-time 
model of network dynamics, where each observed event represents a single 
actor altering his or her outgoing links to optimize an objective function
based on local neighborhood statistics.  \citet{robins:01} have indepedently 
studied a family of models of network dynamics over discrete time steps, 
quite similar to those presented below; in some sense, the present work 
can be viewed as a further exploration of these models, their properties 
and uses.  However, this exploration goes beyond the \citep{robins:01} work, in that
we explore the statistical properties of these models, such as their (non)degeneracy
tendencies, and the quality of fit that these models achieve when applied to real 
network time series data; such properties have not previously been systematically 
investigated for these types of models, though related work has recently been done 
on static ERGMs~\citep{handcock:03}.  We also explore algorithmic issues in 
calculating the MLE estimators and performing hypothesis tests with these models.
Furthermore, we feel that the added flexibility in the parametrization of these models 
below makes them somewhat easier to specify and work with, compared to the 
description in \citep{robins:01}, which although quite elegant, requires the sufficient statistics
to be nondecreasing in the relation indicator variables for the network.

In the following sections, we propose a model family we would like to refer to as temporal ERGM, or TERGM, that is capable of modeling network evolution, while maintaining the flexibility of a fully general ERGM.
Furthermore, these models build upon a generic ERGM formalism, so that
existing methods developed for ERGMs over the past two
decades are readily adapted to apply to the temporal models as
well. We prove that a very general subclass of the TERGM is nondegenerate and explain how to calculate their maximum likelihood estimates from network data. 
Furthermore we show that these models can indeed be fitted to capture signature dynamic properties of real world evolving networks, and can be applied in 
hypothesis testing, nodal classification, and other applications. 

\section{Discrete Temporal Models}

% make a markov assumption between time steps
We begin by describing the basic form of the type of model we study.
Specifically, one way to simplify a statistical model for evolving social networks is to make a Markov assumption
on the network from one time step to the next.  Specifically, if $A^t$ is the weight matrix
representation of a single-relation social network at time $t$, then we might make the assumption that $A^t$ is independent
of $A^1,\ldots,A^{t-2}$ given $A^{t-1}$.  Put another way, a sequence of network observations
$A^1,\ldots,A^t$ has the property that
\begin{equation*}
\P(A^2,A^3,\ldots,A^t|A^1) = \P(A^t|A^{t-1})\P(A^{t-1}|A^{t-2})\cdots\P(A^2|A^1).
\end{equation*}
% specify the conditional as an ERG pmf
With this assumption in mind, we can now set about deciding what the form of the conditional
PDF $\P(A^t|A^{t-1})$ should be.  Given our Markov assumption, one natural way to generalize ERGMs for evolving networks is to
assume $A^t | A^{t-1}$ admits an ERGM representation.  That is, we can specify a function
${\boldsymbol \Psi}: {\mathbb R}_{n \times n}\times{\mathbb R}_{n \times n} \rightarrow {\mathbb R}^k$, which can be understood as a {\it temporal potential} over cliques across two time-adjancent networks, 
and parameter vector ${\boldsymbol \theta} \in {\mathbb R}^k$, such that the conditional PDF has the following form.
\begin{equation}
\P(A^t | A^{t-1}, {\boldsymbol \theta}) = \frac{1}{Z({\boldsymbol \theta},A^{t-1})} \exp\left\{{\boldsymbol \theta}^\prime {\boldsymbol \Psi}(A^t,A^{t-1})\right\}
\label{eqn:tergm}
\end{equation}

We refer to such a model as a TERGM, for Temporal Exponential Random Graph Model.
Note that specifying the joint distribution requires one to specify a distribution
over the first network: $A^1$.  This can generally be accomplished fairly naturally
using an ERGM.  For simplicity of presentation,
we avoid these details in subsequent sections by assuming the distribution over this
initial network is functionally independent of the parameter $\boldsymbol \theta$.

In particular, we will be especially interested in the special case of these models in which
\begin{equation}
{\boldsymbol \Psi}(A^t,A^{t-1}) = \sum_{ij} {\boldsymbol \Psi}_{ij}(A_{ij}^t,A^{t-1}).
\label{eqn:factor}
\end{equation}
This form of the temporal potential function represents situations where the conditional distribution of $A^t$ given $A^{t-1}$ factors over the entries $A_{ij}^t$ of $A^t$.  As we will see, such models possess a number of desirable properties.

\subsection{An Example}
\label{subsec:example}
% example: density, continuity, reciprocity, transitivity
To illustrate the expressiveness of this framework, we present the
following simple example model.  For simplicity, assume the weight matrix of the network is binary (i.e., an adjacency matrix).
Define the following statistics, which represent
\textit{density}, \textit{stability}, \textit{reciprocity}, and \textit{transitivity}, respectively.
\small{\begin{align*}
\Psi_D(A^t,A^{t-1}) & =  \frac{1}{(n\!-\!1)} \sum_{ij} A_{ij}^t \\
\Psi_S(A^t,A^{t-1}) & =  \frac{1}{(n\!-\!1)} \sum_{ij}\left[A_{ij}^t A_{ij}^{t-1}\!+\!(1\!-\!A_{ij}^t)(1\!-\!A_{ij}^{t-1})\right] \\
\Psi_R(A^t,A^{t-1}) & =  n \left[ \sum_{ij} A_{ji}^t A_{ij}^{t-1}\right] \biggm/ \left[ \sum_{ij} A_{ij}^{t-1} \right] \\
\Psi_T(A^t,A^{t-1}) & =  n \left[ \sum_{ijk} A_{ik}^t A_{ij}^{t-1} A_{jk}^{t-1} \right] \biggm/ \left[ \sum_{ijk} A_{ij}^{t-1} A_{jk}^{t-1} \right] \\
\end{align*}}\normalsize{The statistics are each scaled to a constant range (in this case $[0,n]$) to enhance interpretability of the model parameters.
The conditional probability mass function \eqref{eqn:tergm} is
governed by four parameters: $\theta_D$ controls the \emph{density}, or the number of ties in the network as a whole;
$\theta_S$ controls the \emph{stability}, or the tendency of a link that does (or does not) exist at time $t-1$ to continue
existing (or not existing) at time $t$;
$\theta_R$ controls the \emph{reciprocity}, or the tendency of a link from $i$ to $j$ to result in a link from $j$ to $i$ at the next time step;
and $\theta_T$ controls the \emph{transitivity}, or the tendency of a tie from $i$ to $j$ and from $j$ to $k$ to result
in a tie from $i$ to $k$ at the next time step. The transition probability for this temporal network model can then be written as follows.}
\begin{equation*}
\P(A^t | A^{t-1}, {\boldsymbol \theta}) = \frac{1}{Z({\boldsymbol \theta},A^{t-1})} \exp\left\{\sum_{j \in \{D,S,R,T\}}\theta_j \Psi_j(A^t,A^{t-1})\right\}
\end{equation*}

\subsection{More General Models}

For simplicity, we will only discuss the simple
models described above. However, one can clearly extend this
framework to allow multiple relations in the network, actor
attributes, relation attributes, longer-range Markov dependencies, or
a host of other possibilities.  In fact, many of the results below can
easily be generalized to deal with these types of extensions.

\section{Estimation}

The estimation task for models of the form~\eqref{eqn:tergm} is to use the sequence of observed
networks, $N^1,N^2,\ldots,N^T$, to find an estimator $\hat{\boldsymbol \theta}$
that is close to the actual parameter values $\boldsymbol \theta$ in some sensible metric.
As with ERGMs, in general the normalizing constant $Z$ could be computationally intractable, often making explicit solutions
of maximum likelihood estimation difficult.  However, general techniques for MCMC sampling to
enable approximate maximum likelihood estimation for ERGMs have been studied in some depth
and have proven successful for a variety of models \citep{snijders:02}.  By a slight
modification of these algorithms, we can apply the same general techniques as follows.

Let
\begin{equation}
\L({\boldsymbol \theta}; N^1,N^2,\ldots,N^T) = \log\P(N^{2},N^{3},\ldots,N^T | N^1,{\boldsymbol \theta} ),
\label{eqn:L}
\end{equation}
\begin{equation*}
{\mathbf M}(t,{\boldsymbol \theta}) = \mathbb{E}_{\boldsymbol \theta}\left[\Psi(\underline{\mathbf N}^t,N^{t-1})|N^{t-1}\right],
\end{equation*}
\begin{equation*}
{\mathbf C}(t,{\boldsymbol \theta}) = \mathbb{E}_{\boldsymbol \theta}\left[\Psi(\underline{\mathbf N}^t,N^{t-1}) \Psi(\underline{\mathbf N}^t,N^{t-1})^\prime |N^{t-1}\right].
\end{equation*}
where expectations are taken over the random variable $\underline{\mathbf N}^t$, the network at time $t$.  Note that
\begin{equation*}
\nabla \L({\boldsymbol \theta}; N^1,\ldots,N^T) = \sum_{t=2}^T \left(\Psi(N^t,N^{t-1}) - M(t,{\boldsymbol \theta})\right)
\end{equation*}
and
\begin{equation*}
\nabla^2 \L({\boldsymbol \theta}; N^1,\ldots,N^T) = \sum_{t=2}^T \left( M(t,{\boldsymbol \theta}) M(t,{\boldsymbol \theta})^\prime - C(t,{\boldsymbol \theta}) \right).
\end{equation*}
The expectations can be approximated by Gibbs sampling from the conditional distributions \citep{snijders:02},
so that we can perform an unconstrained optimization procedure akin to Newton's method: approximate the
expectations, update parameter values in the direction that increases \eqref{eqn:L}, repeat
until convergence. A related algorithm is described by \citep{geyer:92} for general exponential families, and
variations are given by \citep{snijders:02} that are tailored for ERG models.  The following is
a simple version of such an estimation algorithm.

% (give algorithm)
\begin{bigboxit}
\noindent 1. Randomly initialize ${\boldsymbol \theta}^{(1)}$ \\
2. For $i = 1$ up until convergence \\
3. \;\;  For $t = 2, 3,\ldots, T$ \\
4. \;\;\; \; Sample $\hat{N}_{(i)}^{t,1},\ldots,\hat{N}_{(i)}^{t,B} \sim \P(\underline{\mathbf N}^t | N^{t-1}, {\boldsymbol \theta}^{(i)})$ \\
5. \;\;\;\; \; $\hat{\mu}_{(i)}^t = \frac{1}{B}\sum_{b=1}^B {\boldsymbol \Psi}(\hat{N}_{(i)}^{t,b},N^{t-1})$\\
6. \;\;\;\; \; $\hat{C}_{(i)}^t = \frac{1}{B}\sum_{b=1}^B {\boldsymbol \Psi}(\hat{N}_{(i)}^{t,b},N^{t-1}) {\boldsymbol \Psi}(\hat{N}_{(i)}^{t,b},N^{t-1})^\prime$\\
7. \;\; $\hat{H}_{(i)} = \sum_{t=2}^T [\hat{\mu}_{(i)}^t \hat{\mu}_{(i)}^{t\prime} - \hat{C}_{(i)}^t]$\\
8. \;\; ${\boldsymbol \theta}^{(i+1)} \leftarrow {\boldsymbol \theta}^{(i)} - \hat{H}_{(i)}^{-1} \sum_{t=2}^T\left[{\boldsymbol \Psi}(N^t,N^{t-1}) - \hat{\mu}_{(i)}^t\right]$
\end{bigboxit}

\normalsize{The choice of $B$ can affect the convergence of this algorithm.  Generally, larger $B$ values will give more accurate
updates, and thus fewer iterations needed until convergence. However, in the early stages of the algorithm, precise updates might
not be necessary if the likelihood function is sufficiently smooth, so that a $B$ that grows larger only when more precision is
needed may be appropriate.
If computational resources are limited, it is possible (though less certain) that the algorithm might
still converge even for small $B$ values (see \citep{carreira:05} for an alternative approach to sampling-based MLE, which seems
to remain effective for small $B$ values).}

\subsection{Product Transition Probabilities}

Although the general case~\eqref{eqn:tergm} may often require a
sampling-based estimation procedure such as that given above, it turns
out that the special case of~\eqref{eqn:factor} does not.  In this
case, as long as the ${\boldsymbol \Psi}_{ij}$ functions are
computationally tractable, we can tractably perform \emph{exact}
updates in Newton's method, rather than approximating them with
sampling.

% we may want to explicitly explain how to do exact updates.
% or in fact, we may sometimes be able to give an explicit formula for the MLE

\subsection{Evalutation of Parameter Recovery}

% simulation study with example model for convergence rate demonstration
To examine the convergence rate empirically, we display in Figure~\ref{fig:simulation}
the convergence of this algorithm on data generated from the example model given in Section~\ref{subsec:example}.
The simulated data is generated by sampling from the example model with randomly generated $\boldsymbol \theta$,
and the loss is plotted in terms of Euclidean distance of the estimator from the true parameters.
To generate the initial $N^1$ network, we sample from the pmf
$\frac{1}{Z({\boldsymbol \theta})}\exp\{{\boldsymbol \theta}^\prime {\boldsymbol \Psi}(N^1,N^1)\}$.
The number of actors $n$ is 100.
The parameters are initialized uniformly in the range $[0,10)$, except for $\theta_D$, which is initialized
to $-5\theta_S - 5\theta_R - 5\theta_T$.  This tends to generate networks with reasonable densities.  The
results in Figure~\ref{fig:simulation} represent averages over 10 random initial configurations of the parameters
and data.  In the estimation algorithm used, $B=100$, but increases to $1000$ when the Euclidean distance between
parameter estimates from the previous two iterations is less than $1$.  Convergence is defined as the Euclidean
distance between ${\boldsymbol \theta}^{(i+1)}$ and ${\boldsymbol \theta}^{(i)}$ being within $0.1$.
Since this particular model is simple enough for exact
calculation of the likelihood and derivatives thereof (see above), we also compare against Newton's method with exact updates (rather
than sampling-based).  We can use this to determine how much of the loss is due to the approximations being
performed, and how much of it is intrinsic to the estimation problem.  The parameters returned by the
sampling-based approximation are usually almost identical to the MLE obtained by Newton's method,
and this behavior manifests itself in Figure~\ref{fig:simulation} by the losses being visually indistinguishable.

\begin{figure}
%\scalebox{.5}{\includegraphics{Simulation.eps}}
\setlength{\epsfxsize}{3.5in}
\center{\epsfbox{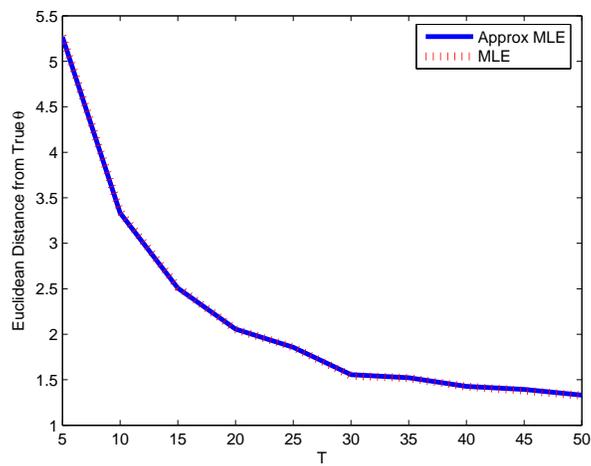}}
\caption{Convergence of estimation algorithm on simulated data, measured in Euclidean distance
of the estimated values from the true parameter values.  The approximate MLE
from the sampling-based algorithm is almost identical to the MLE obtained by direct
optimization.}
\label{fig:simulation}
\end{figure}

\section{(Non)Degeneracy of Temporal ERGMs}

There has been much concern expressed in the literature over certain degeneracy issues
that arise when working with Exponential Random Graph Models.
Specifically, Handcock~\citep{handcock:03} has recently been exploring
the fact that, for many ERGMs, most of the parameter space is
populated by distributions that place almost all of the probability
mass on a small number of networks (typically the complete or empty
graphs).  This leads to several negative effects.  For instance, since
we do not expect the true generating distribution to be degenerate,
the prevalence of these degenerate distributions intuitively reflects
a mismatch between the model and the type of process we wish to
capture with it.  Additionally, it is often the case that degenerate
distributions can be found very close to any nondegenerate
distribution, so that slight variations in the parameters cause the
distribution to become degenerate.  This also leads to another
problem, namely inference degeneracy.  Even if the generating
distribution is modeled by some parameter values with a nondegenerate
distribution, the degeneracy of nearby distributions may prevent
commonly used maximum likelihood estimation techniques from converging
to it within a reasonable sample size; specifically, the aforementioned MCMC techniques may require an
impractically large number of samples, or may even fail to work at
all~\citep{handcock:03}.

One natural question to ask is whether such issues also affect these
temporal extensions.  In the simple case where the transition
distribution factors over the edges, as in~\eqref{eqn:factor}, it
turns out these models avoid such problems entirely.

\subsection{Nondegeneracy of the Example Model}

To keep the initial explanation of this phenomenon as simple as possible, we begin
this discussion by looking at the special case of the example model from Section~\ref{subsec:example}.

For any given entry $A_{ij}^t$, the networks $A^{t-1}$ that minimize
and maximize the probability that $A_{ij}^t=1$ are the empty graph and
complete graph; which one maximizes and which one minimizes it depends
on the parameter values.  If $A^{t-1}$ is the empty graph, then
%(((upper and lower bounds on the probability of an edge)))
\[\mathbb{P}(A_{ij}^t = 1 | A^{t-1}) = \frac{exp\{\theta_{D}/(n-1)\}}{exp\{\theta_{D}/(n-1)\}+exp\{\theta_S/(n-1)\}}.\]
Under $A^{t-1}$ as the complete graph, it is
\[\mathbb{P}(A_{ij}^t = 1 | A^{t-1}) = \frac{exp\{(\theta_{D} + \theta_{S} + \theta_{T} + \theta_{R})/(n-1)\}}{exp\{(\theta_{D} + \theta_{S}+\theta_{T}+\theta_{R})/(n-1)\}+1}.\]

%(((resulting lower bounds on the entropy)))
So the entropy is lower bounded as follows:
\[H(A^{t}) \geq \min\limits_{A^{t-1}} H(A^{t} | A^{t-1}) = \min\limits_{A^{t-1}} \sum_{ij} H(A_{ij}^{t} | A^{t-1}) \geq \sum_{ij} \min\limits_{A^{t-1}} H(A_{ij}^t | A^{t-1}).\]
We can lower bound $\min\limits_{A^{t-1}} H(A_{ij}^t | A^{t-1})$ by the quantity
\[p \ln \frac{1}{p} + (1-p)\ln \frac{1}{1-p},\]
where $p=$\\
\[\frac{exp\{(|\theta_D|\!+\!|\theta_S|\!+\!|\theta_R|\!+\!|\theta_T|)/(n\!-\!1)\}}{exp\{\!(|\theta_D|\!+\!|\theta_S|\!+\!|\theta_R|\!+\!|\theta_T|)/(n\!-\!1)\!\}\!+\!exp\{\!-(|\theta_D|\!+\!|\theta_S|\!+\!|\theta_R|\!+\!|\theta_T|)/(n\!-\!1)\!\}}\]
\[= \frac{1}{exp\{-2(|\theta_D|+|\theta_S|+|\theta_R|+|\theta_T|)/(n-1)\}+1}.\]
($p$ is an upper bound on $P(A_{ij}^t=1|A^{t-1})$, and $1-p$ is a lower bound on it).
So the entropy lower bound is at most $n(n-1)(p \ln (1/p) +  (1-p) \ln (1/(1-p)))$,
and thus as long as $|\theta_D| + |\theta_S|+|\theta_R|+|\theta_T|$ is not too large,
the entropy is guaranteed to be reasonably large.

%(((bounds on the expected number of edges)))
Other than the entropy, we can get a somewhat more intuitive grasp of
this type of nondegeneracy result by bounding the expected number of
nonzero entries in $A^t$.  In particular, a consequence of the above
reasoning is that the expected number of nonzero entries in $A^t$ is
at most
\[n(n-1)\frac{1}{exp\{-2(|\theta_D|+|\theta_S|+|\theta_R|+|\theta_T|)/(n-1)\}+1},\]
and is at least
\[n(n-1)\frac{1}{exp\{2(|\theta_D|+|\theta_S|+|\theta_R|+|\theta_T|)/(n-1)\}+1}.\]
So again, as long as $|\theta_D|+|\theta_S|+|\theta_R|+|\theta_T|$ is not too large,
we are guaranteed a reasonable expected number of nonzero entries in $A^t$.
% We also expect the mode of the distribution of the number 
% of nonzero entries to be near its mean, due to concentration 
% of measure considerations under each conditional 
% distribution, so that we can use Chernoff bounds to argue that 
% the number of nonzero entries is bounded away from the extremes 
% with high probability.

To give an example of the types of entropy values one gets from a
model of this type, in Figure~\ref{fig:entropy} we plot the exact
entropy values for the example model as a function of $\theta_D$ and
$\theta_S$ (with $\theta_R = \theta_T = 0$ to make a two-dimensional
plot possible), and as a function of $\theta_D$ and $\theta_T$ (with
$\theta_R = \theta_S = 0$); other options, such as fixing the unused
parameters to nonzero values, yield similar plots.  Specifically, the
plotted values are the entropy of $A^2$, where each $A^{i}$ is a $n
\times n$ matrix (where $n=7$ in the left plot, and $n=6$ in the right
plot), and $A^1$ is sampled from the basic Bernoulli graph model, in
which each entry $A_{ij}^1$ is an independent Bernoulli random
variable with probability of being $1$ equal to $0.25$.

It is worth briefly mentioning how these plots are generated. Since it
is not computationally tractable to enumerate all graphs for each
parameter setting, we instead calculate it for equivalence classes of
graphs which can be analytically shown to have identical probability
values, and weight each class according to its size in the entropy
calculation.  For the first plot, since the conditional probability of
$A^2$ given $A^1$ is only a function of how many edges are present in
$A^2$ and how many $ij$ values have $A^2_{ij} = A^1_{ij}$, and since
the edges of $A^1$ are exchangeable, we can write the marginal
distribution of $A^2$ purely in terms of the number of edges.  Thus we
need only calculate $n(n-1)$ probability values, and the entropy is a
weighted sum, where the weights are combinatorial quantities
reflecting the number of graphs with that many edges.  For the second
plot, the situation is more complex but the idea is similar.  In this
case, we define the equivalence classes based purely on graph
isomorphisms.  The number of distinct isomorphic networks of six nodes
is 156, a significant reduction from the total number of networks,
rendering the calculation of entropy computationally tractable.
%%% ARE THESE DIRECTED OR UNDIRECTED?

In both plots, small magnitudes of the parameters give distributions
with high entropy, as predicted.

\begin{figure}
\setlength{\epsfxsize}{2.5in}
%\center{\epsfbox{Simulation.eps}}
\epsfbox{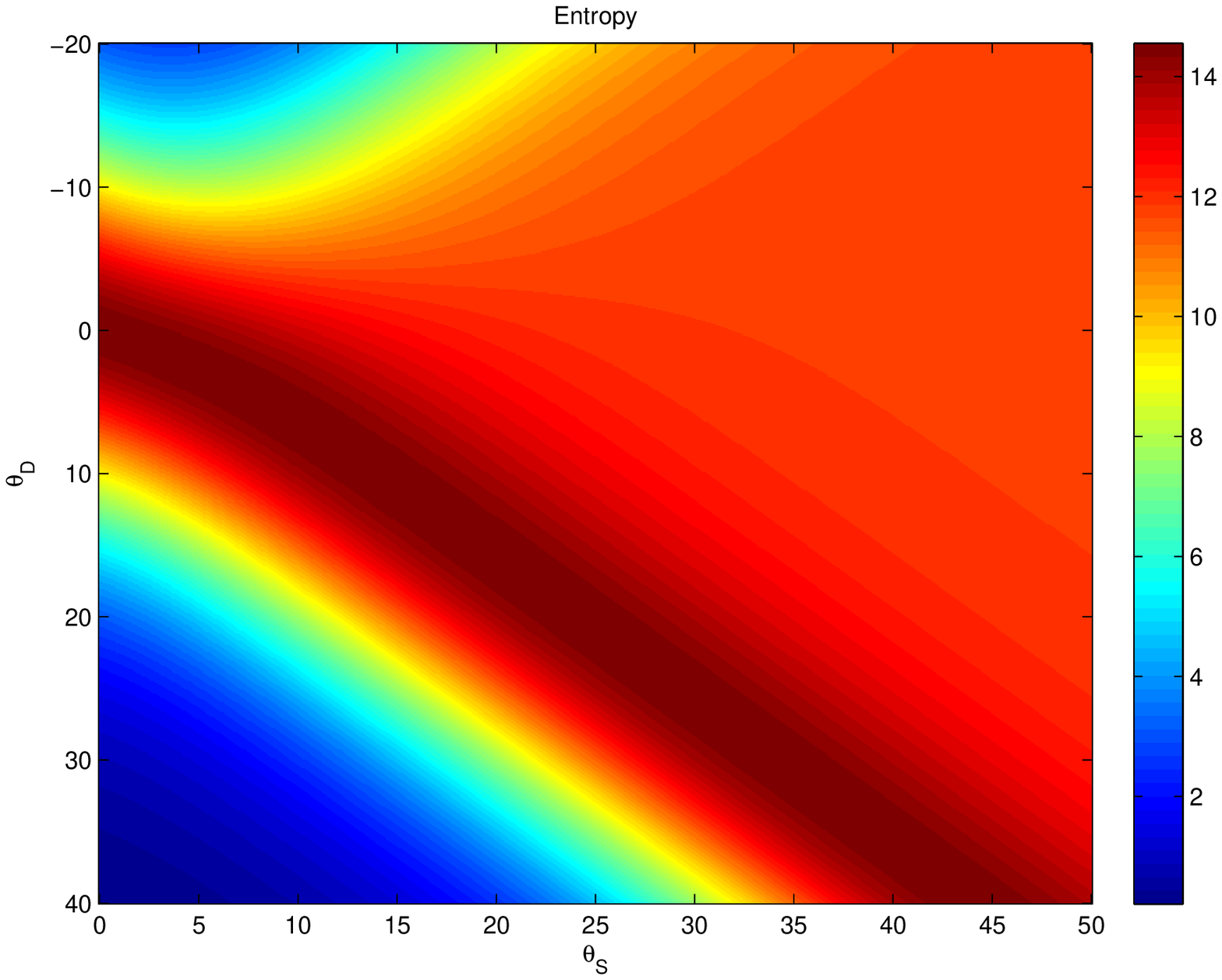}
\setlength{\epsfxsize}{2.5in}
\epsfbox{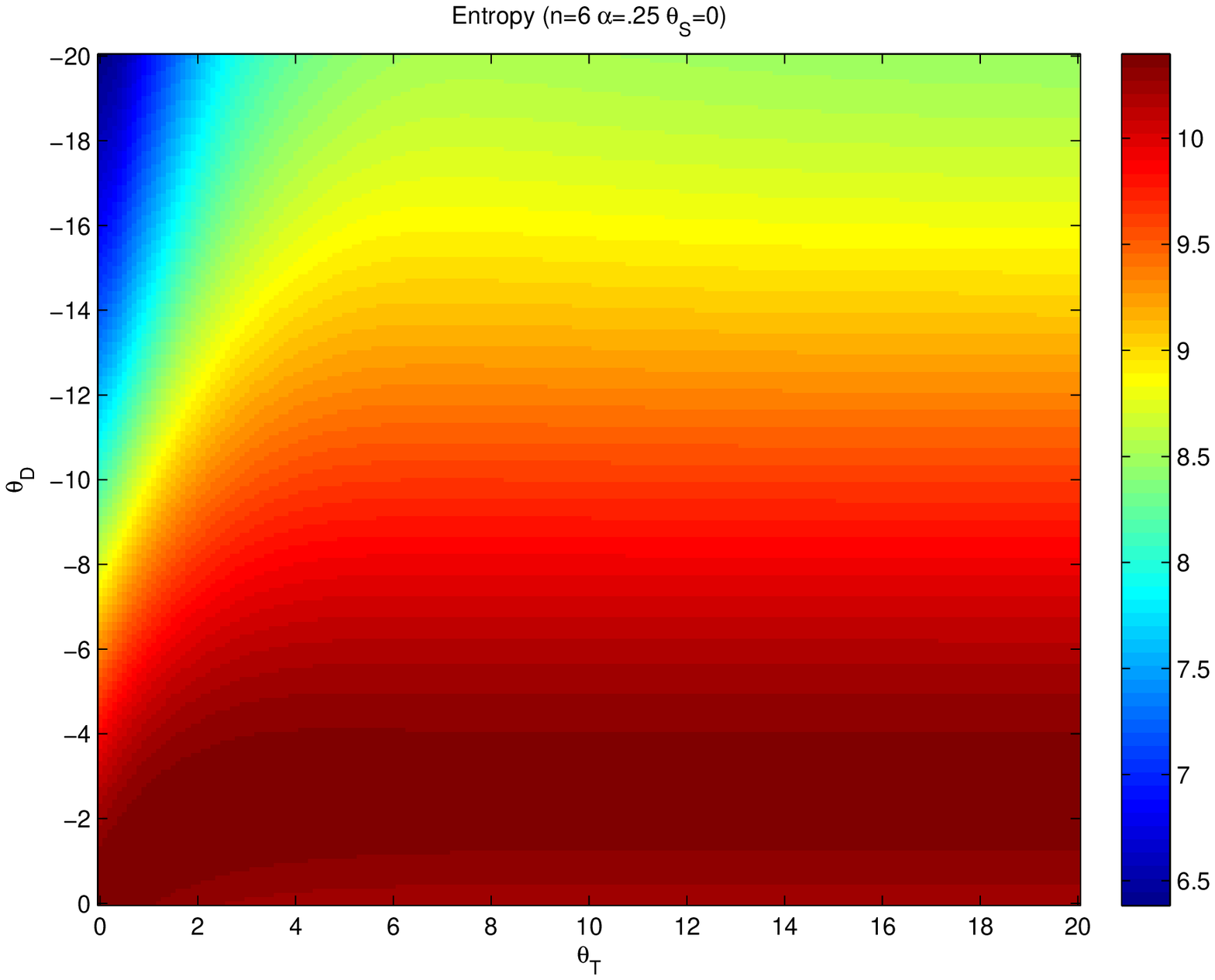}
\caption{Entropy plots for the example model. In both plots, small magnitudes of the parameters give distributions with high entropy, as predicted.}
\label{fig:entropy}
\end{figure}

\subsection{Nondegeneracy Under General Product Transitions}

%(((general formulation)))
We can generalize the preceding discussion beyond the simple example
model, by considering general transition distributions that factor
over entries of $A^t_{ij}$ as follows.
Suppose $\{\Psi_{k}(A^t,A^{t-1})\}_k$ is a sequence of functions such
that $\Psi_{k}(A^t,A^{t-1})$ $= \sum_{ij} \Psi_{ijk}(A_{ij}^t,A^{t-1})$
(i.e., satisfying~\eqref{eqn:factor}), where each $\Psi_{ijk}$ has
range contained in $[-\beta,\beta]$ for some $\beta > 0$, so that the
range of $\Psi_k$ is in $[-n(n-1)\beta,n(n-1)\beta]$.
Then we consider transition models of the form~\eqref{eqn:tergm}, with
these $\Psi$ values.  That is,
\begin{equation}
\P(A^t | A^{t-1}, {\boldsymbol \theta}) = \frac{1}{Z({\boldsymbol \theta},A^{t-1})} \exp\left\{\sum_{k}\theta_k \Psi_k(A^t,A^{t-1})\right\}.
\label{eqn:super-factor}
\end{equation}
Note that these models factor over the entries $A_{ij}^t$ given $A^{t-1}$.

\begin{theorem}
Let \[p = \frac{1}{exp\{2\beta\sum_{k}|\theta_k|\}+1}.\]
For models of the form~\eqref{eqn:super-factor},  
the expected number of nonzero entries in $A^t$ is in the range
\[\left[n(n-1) p , n(n-1)(1-p)\right],\]
and the entropy can be lower bounded as
\[H(A^t) \geq n(n-1)\left(p \log \frac{1}{p} + (1-p) \log \frac{1}{1-p}\right).\]
In particular, as long as $\sum_{k}|\theta_k|$ is not too large, this
bound implies the entropy will be reasonably large.
\end{theorem}
\begin{proof}[Sketch]
%(((upper and lower bounds on the probability of an edge)))
As above, we can upper bound the probability of a particular entry $A_{ij}^t$ taking value $1$, given $A^{t-1}$, by
\[\frac{1}{exp\{-2\beta \sum_{k}|\theta_k|\}+1},\]
and lower bound it by
\[\frac{1}{exp\{2\beta \sum_k |\theta_k|\}+1}.\]
Since the conditional distribution given $A^{t-1}$ factors over the
edges of $A^t$, the expected number of edges given $A^{t-1}$ is in
this range, multiplied by $n(n-1)$.  Since these bounds are
independent of $A^{t-1}$, they also hold for the expectation under the
marginal distribution of $A^t$.  Similarly, as before we can lower
bound the entropy under the marginal distribution of $A^t$ as
\[H(A^t) \geq \sum_{ij} \min_{A^{t-1}} H(A_{ij}^t | A^{t-1}),\]
and due to the aforementioned bounds on the conditional of $A_{ij}^t$,
the quantity $H(A_{ij}^t | A^{t-1})$ is at least
\[p \ln \frac{1}{p} + (1-p) \ln \frac{1}{1-p}.\] %\qed
%In particular, the expected number of nonzero entries in $A^t$ is at most
%\[n(n-1)\frac{1}{exp\{-2\beta \sum_{k}|\theta_k|\}+1},\]
%and at least
%\[n(n-1)\frac{1}{exp\{2\beta\sum_k |\theta_k|\}+1}.\]
%Also, the entropy can be lower bounded, as before, by
%\[H(A^t) \geq \sum_{ij} \min_{A^{t-1}} H(A_{ij}^t | A^{t-1}),\]
%and the quantity $H(A_{ij}^t | A^{t-1})$ is at least
%\[p \ln \frac{1}{p} + (1-p) \ln \frac{1}{1-p},\]
%where
%\[p = \frac{1}{exp\{-2\beta\sum_{k}|\theta_k|\}+1}.\]
%In particular, as long as $\sum_{k}|\theta_k|$ is not too large, this
%bound implies the entropy will be reasonably large.
\end{proof}

\section{Hypothesis Testing: A Case Study}
\label{sec:testing}

% this section gives an example of how one might use these types of models
As an example of how models of this type might be used in practice,
we present a simple hypothesis testing application.  Here we see the
generality of this framework pay off, as we can use models of this
type to represent a broad range of scientific hypotheses.  The general
approach to hypothesis testing in this framework is first to write
down potential functions representing transitions one expects to
be of some significance in a given population, next to write down
potential functions representing the usual ``background'' processes
(to serve as a null hypothesis), and third to plug these potentials into
the model, calculate a test statistic, and compute a p-value.

% introduce the data and the informal hypotheses
The data involved in this example come from the United States
108$^{th}$ Senate, having $n=100$ actors.  Every time a proposal is
made in the Senate, be it a bill, amendment, resolution, etc., a
single Senator serves as the proposal's \emph{sponsor} and there may
possibly be several \emph{cosponsors}.  Given records of all proposals
voted on in the full Senate, we create a sliding window of 100
consecutive proposals.  For a particular placement of the window, we
define a binary directed relation existing between two Senators if and
only if one of them is a sponsor and the other a cosponsor for the
same proposal within that window (where the direction is toward the
sponsor).  The data is then taken as evenly spaced snapshots of this
sliding window, $A^1$ being the adjacency matrix for the first 100
proposals, $A^2$ for proposal 31 through 130, and so on shifting the
window by 30 proposals each time\ignore{\footnote{This overlap will artificially inflate the stability parameter.
However, that parameter value is not essential for this test.
The benefit of having a window size (100) larger than the step size (30) is that it prevents having too small of a step size from having a dramatic influence.
If we were interested in more accurately modeling the intuitive notion of stability, one option would be to take greater care in selecting the step size so that
window size need not be larger than the step size;
another option would be to use a higher-order model (i.e., $K>1$) for the reciprocity potentials.}}.  In total,
there are 14 observed networks in this series, corresponding to the
first 490 proposals addressed in the 108$^{th}$ Senate.

% introduce the formal null and alternative hypotheses
In this study, we propose to test the hypothesis that intraparty
reciprocity is inherently stronger than interparty reciprocity.  To
formalize this, we use a model similar to the example given
previously.  The main difference is the addition of party membership
indicator variables.  Let $P_{ij} = 1$ if the $i^{th}$ and $j^{th}$
actors are in the same political party, and 0 otherwise, and let
$\bar{P}_{ij} = 1-P_{ij}$.  Define the following potential functions,
representing \emph{stability}, \emph{intraparty density},
\emph{interparty density},\footnote{We split density to intra- and
inter-party terms so as to factor out the effects on reciprocity of
having higher intraparty density.}
\emph{overall reciprocity}, \emph{intraparty reciprocity}, and \emph{interparty reciprocity}.

\small{\begin{align*}
& \Psi_S(A^t,A^{t-1}) & = & \frac{1}{(n\!-\!1)} \sum_{ij}\left[A_{ij}^t A_{ij}^{t-1}\!+\!(1\!-\!A_{ij}^t)(1\!-\!A_{ij}^{t-1})\right] \\
& \Psi_{WD}(A^t,A^{t-1}) & = & \frac{1}{(n\!-\!1)} \sum_{ij} A_{ij}^t P_{ij} \\
& \Psi_{BD}(A^t,A^{t-1}) & = & \frac{1}{(n\!-\!1)} \sum_{ij} A_{ij}^t \bar{P}_{ij} \\
& \Psi_R(A^t,A^{t-1}) & = & n \left[ \sum_{ij} A_{ji}^t A_{ij}^{t-1}\right] \biggm/ \left[ \sum_{ij} A_{ij}^{t-1} \right] \\
& \Psi_{WR}(A^t,A^{t-1}) & = & n \left[ \sum_{ij} A_{ji}^t A_{ij}^{t-1} P_{ij}\right] \biggm/ \left[ \sum_{ij} A_{ij}^{t-1} P_{ij} \right] \\
& \Psi_{BR}(A^t,A^{t-1}) & = & n \left[ \sum_{ij} A_{ji}^t A_{ij}^{t-1} \bar{P}_{ij}\right] \biggm/ \left[ \sum_{ij} A_{ij}^{t-1} \bar{P}_{ij}\right] \\
\end{align*}}

\normalsize{The null hypothesis supposes that the reciprocity observed in this data is the result of
an overall tendency toward reciprocity amongst the Senators, regardless of party.
The alternative hypothesis supposes that there is a stronger tendency toward reciprocity
among Senators within the same party than among Senators from different parties.
Formally, the transition probability for the null hypothesis can be written as }

\begin{equation*}
\P_0(A^t | A^{t-1}, {\boldsymbol \theta}^{(0)}) = \frac{1}{Z_0({\boldsymbol \theta}^{(0)},A^{t-1})} \exp\left\{\sum_{j \in \{S,WD,BD,R\}}\theta^{(0)}_j \Psi_j(A^t,A^{t-1})\right\},
\end{equation*}
while the transition probability for the alternative hypothesis can be written as
\begin{equation*}
\P_1(A^t | A^{t-1}, {\boldsymbol \theta}^{(1)}) = \frac{1}{Z_1({\boldsymbol \theta}^{(1)},A^{t-1})} \exp\left\{{\hskip -.8cm}\sum_{{\hskip .8cm}j \in \{S,WD,BD,WR,BR\}}{\hskip -.8cm}\theta^{(1)}_j \Psi_j(A^t,A^{t-1})\right\}.
\end{equation*}

% describe how hypothesis testing is done with 
% these hypotheses and give the results
For our test statistic, we use the likelihood ratio.
To compute this, we compute the maximum likelihood estimators for each of these models,
and take the ratio of the likelihoods.  For the null hypothesis, the MLE is
\[(\hat{\theta}^{(0)}_S = 336.2, \hat{\theta}^{(0)}_{WD} = -58.0, \hat{\theta}^{(0)}_{BD} = -95.0, \hat{\theta}^{(0)}_R = 4.7)\]
with likelihood value of $e^{-9094.46}$.
For the alternative hypothesis, the MLE is
\[(\hat{\theta}^{(1)}_S = 336.0, \hat{\theta}^{(1)}_{WD} = -58.8, \hat{\theta}^{(1)}_{BD} = -94.3,\hat{\theta}^{(1)}_{WR} = 4.2, \hat{\theta}^{(1)}_{BR} = 0.03)\]
with likelihood value of $e^{-9088.96}$.
The likelihood ratio statistic (null likelihood over alternative
likelihood) is therefore about $0.0041$.  Because the null hypothesis
is composite, determining the p-value of this result is a bit more
tricky, since we must determine the probability of observing a
likelihood ratio at least this extreme under the null hypothesis for
the parameter values ${\boldsymbol \theta}^{(0)}$ that \emph{maximize}
this probability.
That is, \[\text{p-value} = \sup_{{\boldsymbol \theta}^{(0)}} \P_{0}\left\{ \frac{\sup_{\hat{\boldsymbol \theta}^{(0)}}\P_{0}(A^1,\ldots,A^{14}|\hat{\boldsymbol \theta}^{(0)})}{\sup_{\hat{\boldsymbol \theta}^{(1)}}\P_{1}(A^1,\ldots,A^{14}|\hat{\boldsymbol \theta}^{(1)})} \leq 0.0041 \Bigg| {\boldsymbol \theta}^{(0)} \right\}.\]
In general this seems not to be tractable to analytic solution, so we
employ a genetic algorithm to perform the unconstrained optimization,
and approximate the probability for each parameter vector by sampling.
That is, for each parameter vector ${\boldsymbol \theta}^{(0)}$ (for
the null hypothesis) in the GA's population on each iteration, we
sample a large set of sequences from the joint distribution. For each
sequence, we compute the MLE under the null hypothesis and the MLE
under the alternative hypothesis, and then calculate the likelihood
ratio and compare it to the observed ratio.  We calculate the
empirical frequency with which the likelihood ratio is at most
$0.0041$ in the set of sampled sequences for each vector ${\boldsymbol
\theta}^{(0)}$, and use this as the objective function value in the
genetic algorithm.  Mutations consist of adding Gaussian noise (with
variance decreasing on each iteration), and recombination is performed
as usual.  Full details of the algorithm are omitted for brevity (see
\citep{mitchell:96} for an introduction to GAs).  The resulting
approximate p-value we obtain by this optimization procedure is 0.024.

% explain that in general likelihoods are not so 
% easy to compute, but can use upper/lower,lower/upper 
% bound technique (give appropriate citations)
This model is nice in that, because it has the form~\eqref{eqn:factor},
we can compute the likelihoods and derivatives thereof analytically.  
In particular, in models of this form, we can compute
likelihoods and perform Newton-Raphson optimization directly, without
the need of sampling-based approximations.  However, in general this
might not be the case.  For situations in which one cannot tractably
compute the likelihoods, an alternative possibility is to use bounds
on the likelihoods.  Specifically, one can obtain an upper bound on
the likelihood ratio statistic by dividing an upper bound on the null
likelihood by a lower bound on the alternative likelihood.  When
computing the p-value, one can use a lower bound on the ratio by
dividing a lower bound on the null likelihood by an upper bound on the
alternative likelihood.  See \citep{opper:01,wainwright:05} for
examples of how such bounds on the likelihood can be tractably
attained, even for intractable models.

% describe motif transitions idea for specifying a model
In practice, the problem of formulating an appropriate model to encode
one's hypothesis is not well-posed.  One general approach which seems
intuitively appealing is to write down the types of motifs or patterns
one expects to find in the data, and then specify various other
patterns which one believes those motifs could likely transition to
(or would likely \emph{not} transition to) under the alternative
hypothesis.  For example, perhaps one believes that densely connected
regions of the network will tend to become more dense and clique-like
over time, so that one might want to write down a potential
representing the transition of, say, k-cliques to more densely
connected structures.

\section{Classification: A Case Study}

One can additionally consider using these temporal models for
classification.  \\Specifically, consider a transductive learning
problem in which each actor has a static class label, but the learning
algorithm is only allowed to observe the labels of some random subset
of the population.  The question is then how to use the known label
information, in conjunction with observations of the network evolving
over time, to accurately infer the labels of the remaining actors
whose labels are unknown.

As an example of this type of application, consider the alternative
hypothesis model from the previous section (model 1), in which each
Senator has a class label (party affiliation).  We can slightly modify
the model so that the party labels are no longer constant, but random
variables drawn independently from a known multinomial distribution.
Assume we know the party affiliations of a randomly chosen 50
Senators.  This leaves 50 Senators with unknown affiliations.  If we
knew the parameters $\boldsymbol \theta$, we could predict these 50
labels by sampling from the posterior distribution and taking the mode
for each label.  However, since \emph{both} the parameters \emph{and}
the 50 labels are unknown, this is not possible.  Instead, we can
perform Expectation Maximization to \emph{jointly} infer the maximum
likelihood estimator $\hat{\boldsymbol \theta}$ for $\boldsymbol
\theta$ \emph{and} the posterior mode given $\hat{\boldsymbol \theta}$.

Specifically, let us assume the two class labels are \emph{Democrat}
and \emph{Republican}, and we model these labels as independent
Bernoulli(0.5) random variables.  The distribution on the network
sequence given that all 100 labels are fully observed is the same as
given in the previous section.  Since one can compute likelihoods in
this model, sampling from the posterior distribution of labels given
the network sequence is straightforward using Gibbs sampling.  We can
therefore employ a combination of MCEM and Generalized EM algorithms
(call it MCGEM)
\citep{mclachlan:97} with this model to infer the party labels as follows.
In each iteration of the algorithm, we sample from the posterior
distribution of the unknown class labels under the current parameter
estimates given the observed networks and known labels, approximate
the expectation of the gradient and Hessian of the log likelihood
using the samples, and then perform a single Newton-Raphson update
using these approximations.

We run this algorithm on the 108$^{th}$ Senate data from the previous
section.
We randomly select 50 Senators whose labels are observable,
and take the remaining Senators as having unknown labels.
As mentioned above, we assume all Senators are either
Democrat or Republican; Senator Jeffords, the only independent
Senator, is considered a Democrat in this model.
We run the MCGEM algorithm described above to infer the maximum
likelihood estimator $\hat{\boldsymbol \theta}$ for ${\boldsymbol
\theta}$, and then sample from the posterior distribution over the 50
unknown labels under that maximum likelihood distribution, and take
the sample mode for each label to make a prediction.

The predictions of this algorithm are correct on 70\%
of the 50 Senators with unknown labels.
Additionally, it is interesting to note that the parameter values
the algorithm outputs
$(\hat{\theta}_S = 336.0, \hat{\theta}_{WD} = -59.7, \hat{\theta}_{BD} = -96.0, \hat{\theta}_{WR} = 3.8, \hat{\theta}_{BR} = 0.28)$
are very close (Euclidean distance $2.0$) to the maximum likelihood
estimator obtained in the previous section (where all class labels were known).
Compare the above accuracy score with a baseline predictor
that always predicts Democrat, which would get 52\% correct for this train/test split,
indicating that this statistical model of network evolution provides at least a somewhat
reasonable learning bias.  However, there is clearly room for improvement in the model
specification, and it is not clear whether modeling the evolution of the graph is actually of any benefit
for this particular example.  For example, after collapsing this sequence of networks into
a single weighted graph with edge weights equal to the sum of edge weights over
all graphs in the sequence, running Thorsten Joachims' Spectral Graph Transducer algorithm \citep{joachims:03}
gives a 90\% prediction accuracy on the Senators with unknown labels.
These results are summarized in Table~\ref{tbl:class}.
Further investigation is needed into what types of problems can benefit from explicitly modeling
the network evolution, and what types of models are most appropriate for basing a learning bias on.

\begin{table}[h]
\begin{center}
\begin{tabular}{|l|r|}
\hline
Method & Accuracy \\ \hline
Baseline & 52\% \\ \hline
Temporal Model & 70\% \\ \hline
SGT & 90\% \\ \hline
\end{tabular}
\caption{Summary of classification results.}\label{tbl:class}
\end{center}
\end{table}

\section{Assessing Statistic Importance and Quality of Fit: A Case Study}

In this section, we use TERGMs to model the network transitions of the 
108$^{th}$ U.S. Senate network, described in Section~\ref{sec:testing}. 
The dynamic network has 100 nodes and 12 time points\footnote{We have removed the first two time points 
from the original 14-step series of Section~\ref{sec:testing}, 
due to outlier behavior in the initial two time points.  This behavior is explained by 
an initial surge in activity when the Senate reconvenes after a vacation, but is not
part of the usual ``stationary'' behavior, so we chose to exclude it when evaluating 
the quality of fit.}.
We perform two types of experiments here: the first is simply to assess which 
statistics are important for modeling the network transitions, by observing the 
magnitudes of the estimated parameters, and the second assesses the quality of 
fit of a model with a cross-validation experiment.

We start with including three statistics: \emph{Density},
\emph{Stability}, and \emph{Reciprocity}.  The estimated
parameters are plotted in Figure~\ref{fig:theta_3ftr}. We have estimated 11
sets of model parameters, so each box plot in a subplot contains 11
values. Judging by the magnitudes of the weights, we can see
that \emph{Density} and \emph{Stability} play big roles, whereas
\emph{Reciprocity} plays a minor role in this case.

\begin{figure}[tbh!]
\centerline{\includegraphics[scale=0.6]{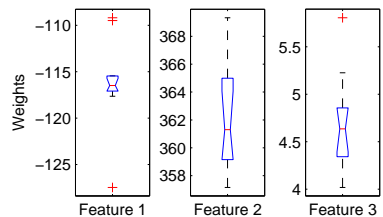}}
\caption{Estimated parameter values (weights) for a TERGM with 3 statistics (features).}
\label{fig:theta_3ftr}
\end{figure}

The 3-statistic TERGM model is so simple that we would not expect it to
have enough modeling power. Therefore, we introduce some 3-node
statistics \footnote{We call \emph{Density}, \emph{Stability}, and
\emph{Reciprocity} 2-node features because their decomposed forms only
involve two nodes.} and expand to include 7 statistics in the model. The 4
new statistics are \emph{Transitivity}, \emph{Reverse-Transitivity},
\emph{Co-Supported}, and \emph{Co-Supporting}, which are illustrated
in Figure~\ref{fig:3NodeFeatures}. \emph{Transitivity} has been
explained in earlier sections. \emph{Reverse-Transitivity} means that if
person B supports person C and person C supports person A at time
$(t-1)$, then it is likely that person A will support person B at time
$t$. \emph{Co-Supported} says that if both A and B are supported by a
third person at time $t$, then it is likely that person A will support
person B at time $t$. \emph{Co-Supporting} is defined similarly. In
all of the cases, we are looking at the influence from the previous
time point on the link from A to B at time $t$.

More formally, the new statistics are defined as
\begin{align*}
\Psi_{RT}(A^t,A^{t-1}) & = n \bigg[\sum_{ijk}A_{jk}^{t-1}A_{ki}^{t-1}A_{ij}^t\bigg] / \bigg[\sum_{ijk}A_{jk}^{t-1}A_{ki}^{t-1} \bigg] \\
\Psi_{CSd}(A^t,A^{t-1}) & = n \bigg[\sum_{ijk}A_{ki}^{t-1}A_{kj}^{t-1}A_{ij}^t\bigg] / \bigg[\sum_{ijk}A_{ki}^{t-1}A_{kj}^{t-1} \bigg] \\
\Psi_{CSg}(A^t,A^{t-1}) & = n \bigg[\sum_{ijk}A_{ik}^{t-1}A_{jk}^{t-1}A_{ij}^t\bigg] / \bigg[\sum_{ijk}A_{ik}^{t-1}A_{jk}^{t-1} \bigg]
\end{align*}.

Figure~\ref{fig:theta_7ftr} shows the parameter values for the TERGM
model with these 7 statistics. Among the 4 new statistics,
\emph{Transitivity} and \emph{Co-Supporting} are major contributors,
while \emph{Reverse-Transitivity} and \emph{Co-Supported} are
neglectable. The effectiveness of \emph{Transitivity} is intuitive,
but the big contrast between the weights for \emph{Co-Supporting} and
for \emph{Co-Supported} could be intricate. It can be explained by the
nature of the data.
%In the senate, only a few members serve as a sponsor to 
% propose a bill/amendment/resolution, while most of the 
% senators serve as a co-sponsor at least once in every period. 
% In fact, 57 members are supported by someone in a period on 
% average while the number for senators who support at least one 
% another is 98.25.
Each proposal has a single sponsor and possibly multiple
co-sponsors. Therefore, each Senator is likely to be a sponsor
(supported by others) for few proposals, while she or he could potentially
be a co-sponsor (supporter) for many more proposals. When the
\emph{Co-Supporting} case happens, it is likely that the two Senators
supported a third Senator on the same proposal, which suggests a shared
position on the issue, which could further lead to a cooperation on
another proposal at a later time. In contrast, when the
\emph{Co-Supported} situation happens, it is certain that the two
Senators are supported by a third senator on different
proposals.\footnote{Links corresponding to a proposal are pointing to
a single node, since there is only one sponsor for each proposal.} 
Although they are co-sponsored by a same Senator, these proposals can
be in very different areas, which does not necessarily suggest a common 
interest for the two sponsors.

\begin{figure}[tbh!]
\centerline{\includegraphics[scale=0.6]{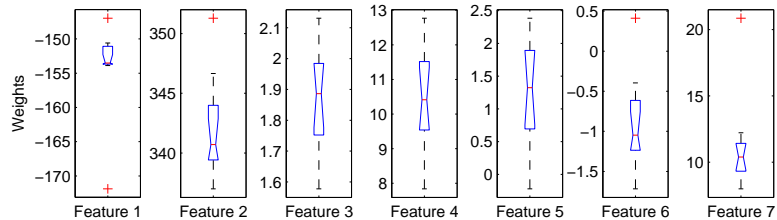}}
\caption{Estimated parameter values (weights) for a TERGM with 7 statistics (features).}
\label{fig:theta_7ftr}
\end{figure}

Next, we add two more 3-node statistics, \emph{Popularity} and
\emph{Generosity}, to the model to have a set of 9 statistics. As
illustrated in Figure~\ref{fig:3NodeFeatures}, \emph{Popularity} says
that if one has a supporter, she or he is likely to have another
supporter. \emph{Generosity} can be understood in the same
manner. More formally, they are defined as
\begin{align*}
\Psi_{P}(A^t,A^{t-1}) & = n \bigg[\sum_{ijk}A_{kj}^{t-1}A_{ij}^t\bigg] / \bigg[\sum_{ijk}A_{kj}^{t-1} \bigg] \\
\Psi_{G}(A^t,A^{t-1}) & = n \bigg[\sum_{ijk}A_{ik}^{t-1}A_{ij}^t\bigg] / \bigg[\sum_{ijk}A_{ik}^{t-1} \bigg]
\end{align*}.

\begin{figure}[tbh!]
\centerline{\includegraphics[scale=0.6]{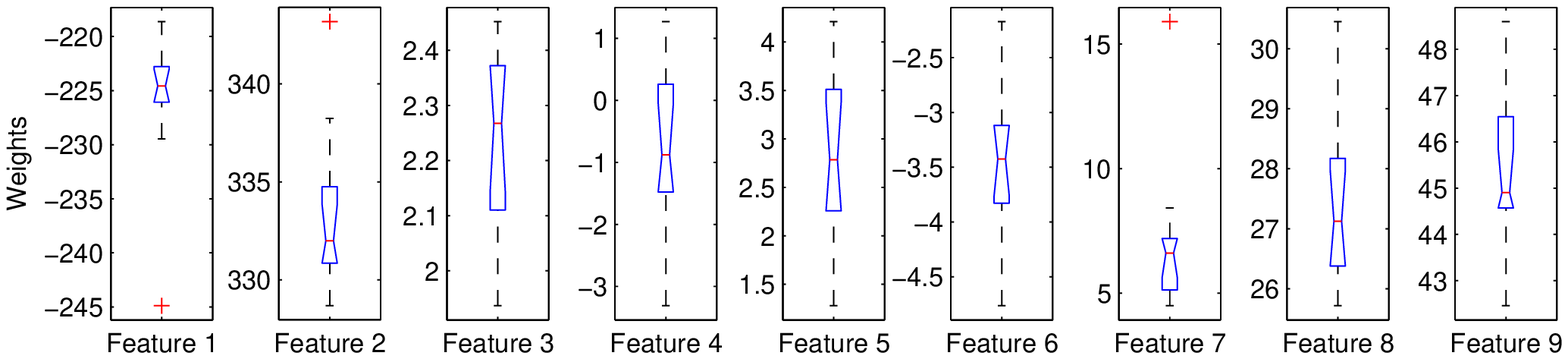}}
\caption{Estimated parameter values (weights) for a TERGM with 9 statistics (features).}
\label{fig:theta_9ftr}
\end{figure}

\begin{figure}[tbh!]
\centerline{\includegraphics[scale=0.5]{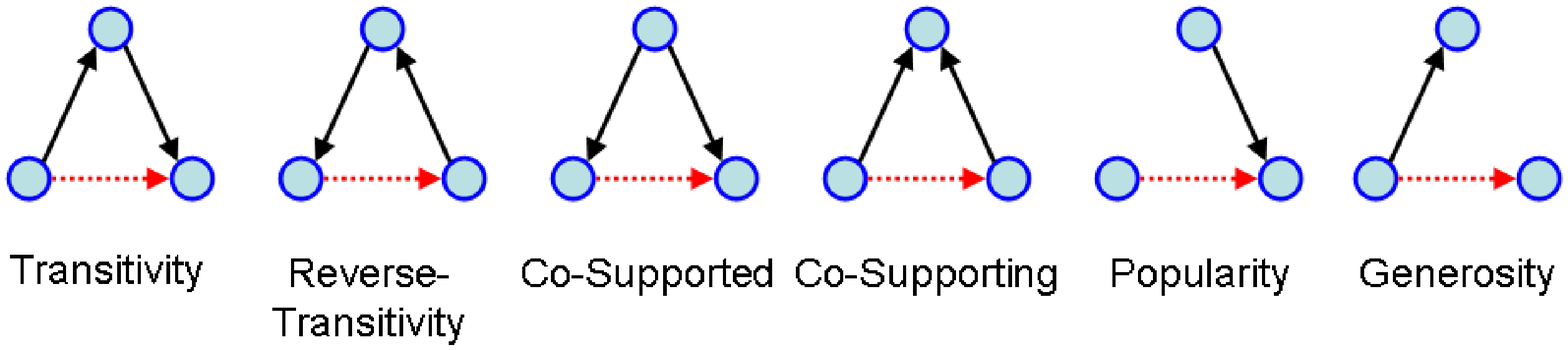}}
\caption{Graph illustrations of six 3-node statistics corresponding to Features 4$-$9 in Figure~\ref{fig:theta_7ftr} and \ref{fig:theta_9ftr}. 
Blue circles are nodes; black solid arrows represent links (or a supporting relationship) at time $(t-1)$; 
red dotted arrows represent an edge at time $t$.}
\label{fig:3NodeFeatures}
\end{figure}

Figure~\ref{fig:theta_9ftr} shows the estimated parameter values for the TERGM
with 9 statistics. Both \emph{Popularity} and \emph{Generosity} have
significant weights, while \emph{Transitivity} has a weight with low magnitude.
In some sense, the \emph{Transitivity}
statistic can be viewed as \emph{Popularity} plus \emph{Generosity} with
some constraints.  This plot indicates that the importance of \emph{Popularity} 
and \emph{Generosity} was the indirect cause of our emphasis on \emph{Transitivity} 
in the simpler models, since in a model containing these statistics, 
\emph{Transitivity} becomes irrelevant.

Next, we heuristically evaluate the quality of fit of the model using a cross-validation style
experiment, as follows.
For each of the time points $t$ (except $t=1$), we estimate a set of parameters
for a TERGM to fit all the observed transitions except the transition from time
point $t-1$ to time point $t$. We then sample 
a number of networks from the conditional distribution over the network at 
time $t$ given the observed network at time $t-1$, under the estimated parameters.
Finally, we compare the $\Psi(A^t,A^{t-1})$ statistic values from the sampled $A^t$ networks
to the $\Psi(A^t,A^{t-1})$ statistic values from the true $A^t$ network.
We repeat this entire process for each $t > 1$ to generate our plots.
The results of this comparison reflect relatively how well TERGMs model the
transitions, given that we are committed to a Markov assumption for the transitions.

Figure~\ref{fig:senate_n3} presents the comparison of statistic values
between ground-truth and sampled networks from the estimated TERGMs from
the cross-validation process described above.  For a few statistics, 
the blue lines lie within the green lines (i.e., in the range of sampled networks) for
most time points, which means that the model does a fairly good job of
predicting the change in statistic values: for example, \emph{Reciprocity}, \emph{Reverse-Transitivity},
\emph{Co-Supported}, \emph{Co-Supporting}, and \emph{Popularity}. It
is worth noting that we can even capture some sharp changes with these models:
for instance, \emph{Reverse-Transitivity}, \emph{Co-Supporting}, and \emph{Popularity} at time point 7 
(the last of these is particularly dramatic). This is somewhat surprising, 
as it intuitively seems like sharp changes might be quite difficult to predict 
with a Markov assumption and such simple statistics.  On the other hand, not all 
statistics are predicted well, such as \emph{Stability}, where the predictions are quite poor;
as such, there is clearly room for improvement in the design of statistics to more accurately 
model time-evolving networks.

\begin{figure}[tbh!]
\def\scA{0.28}
\centerline{\includegraphics[scale=\scA]{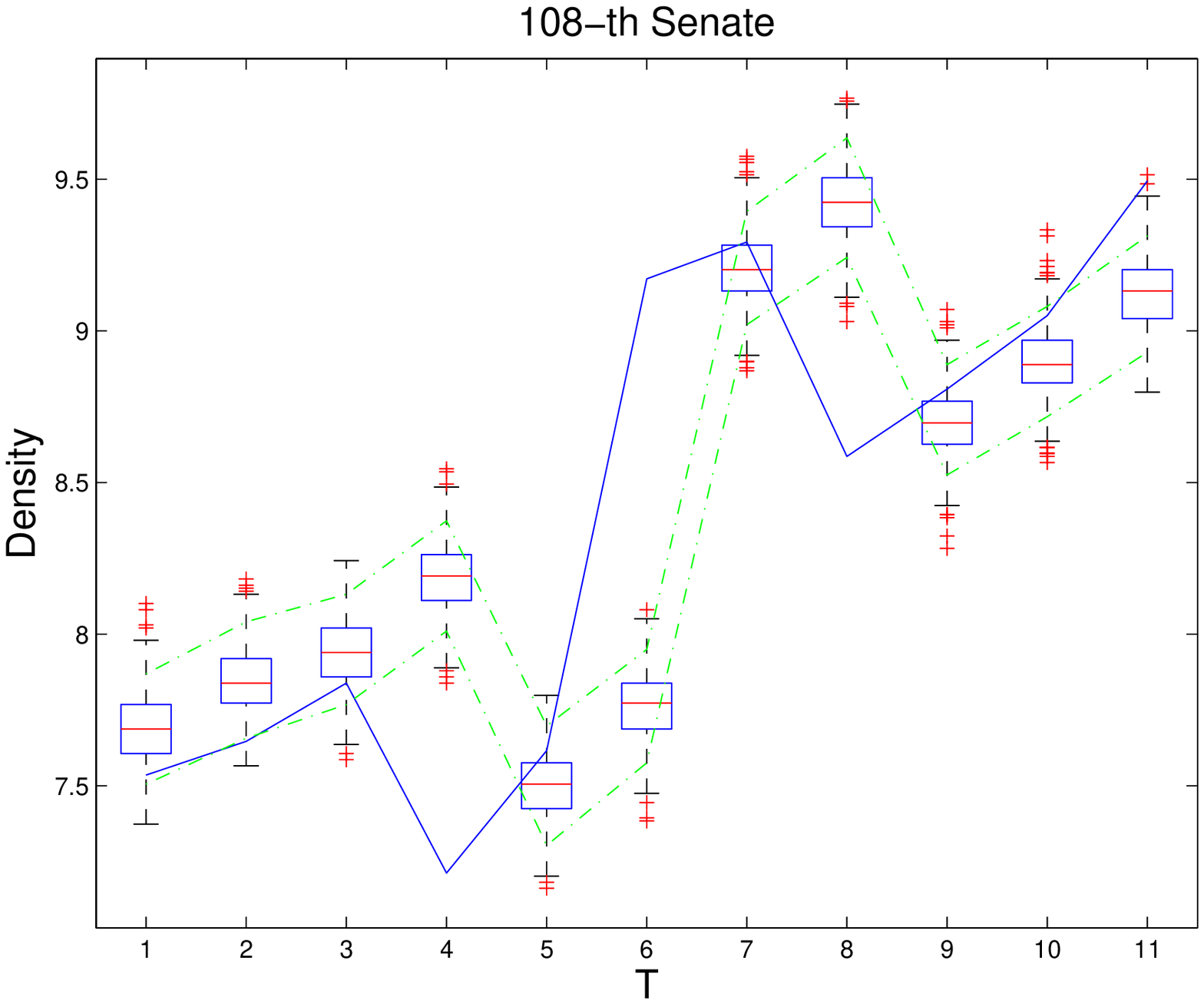}
\includegraphics[scale=\scA]{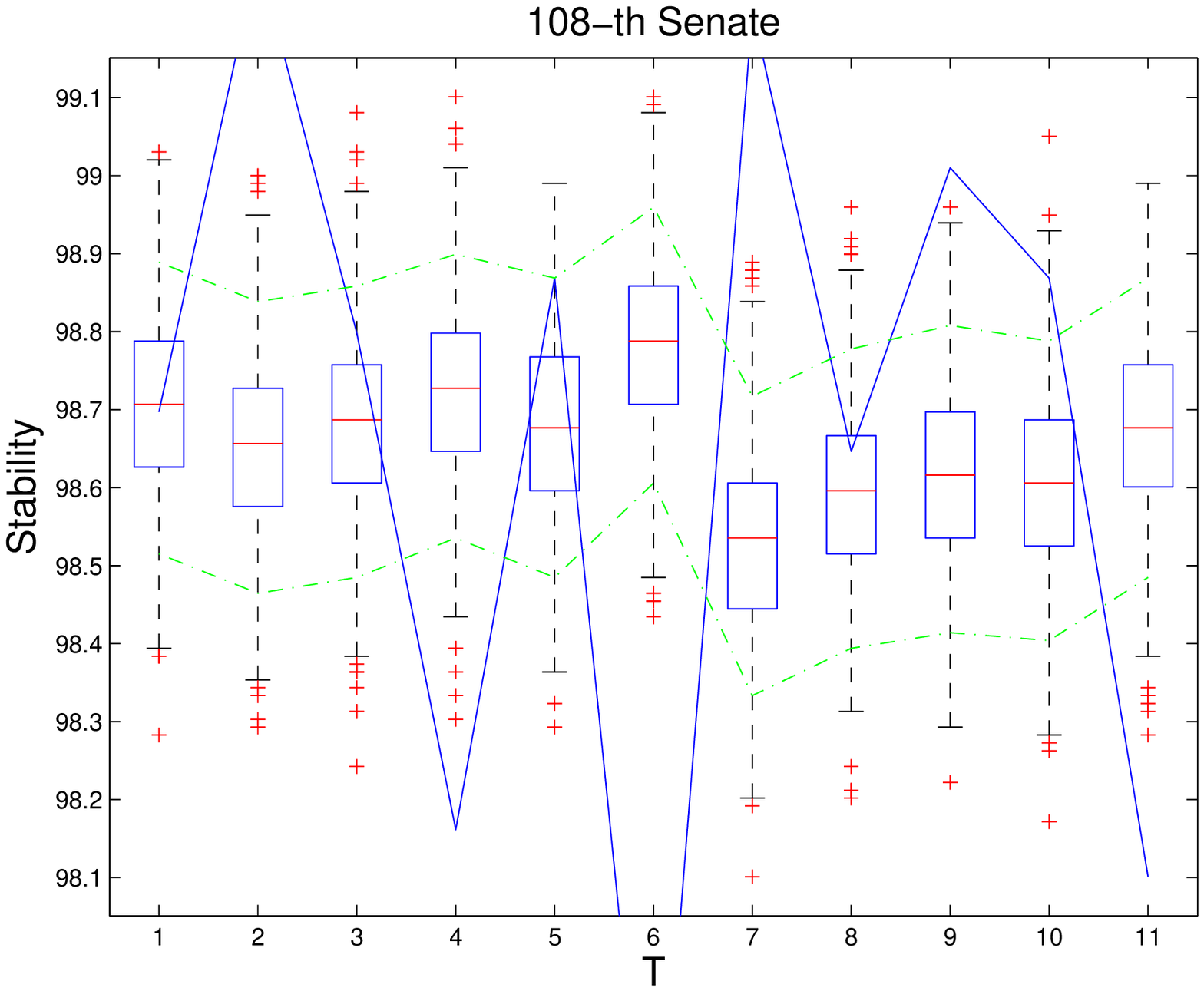}
\includegraphics[scale=\scA]{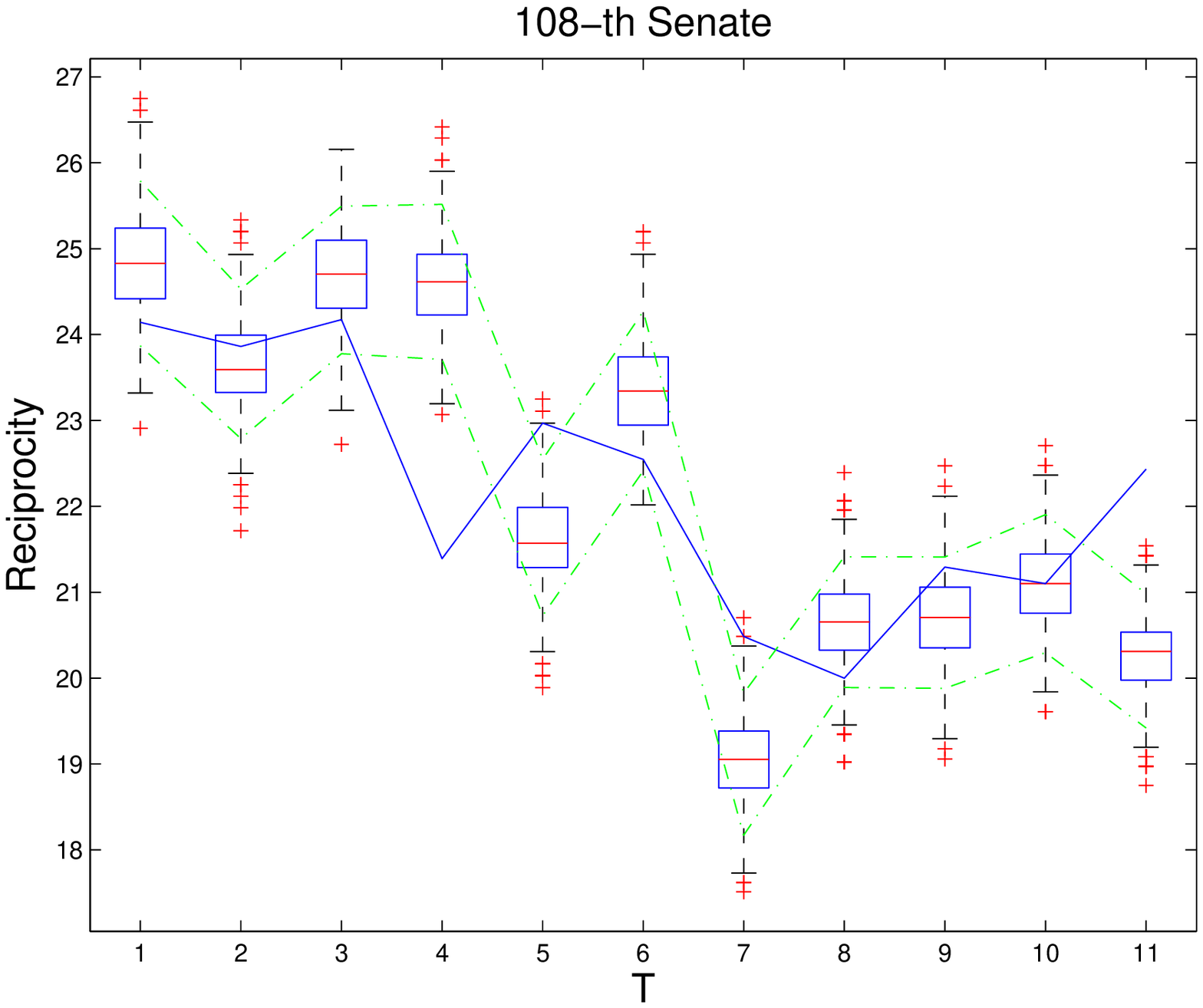}}
\centerline{\includegraphics[scale=\scA]{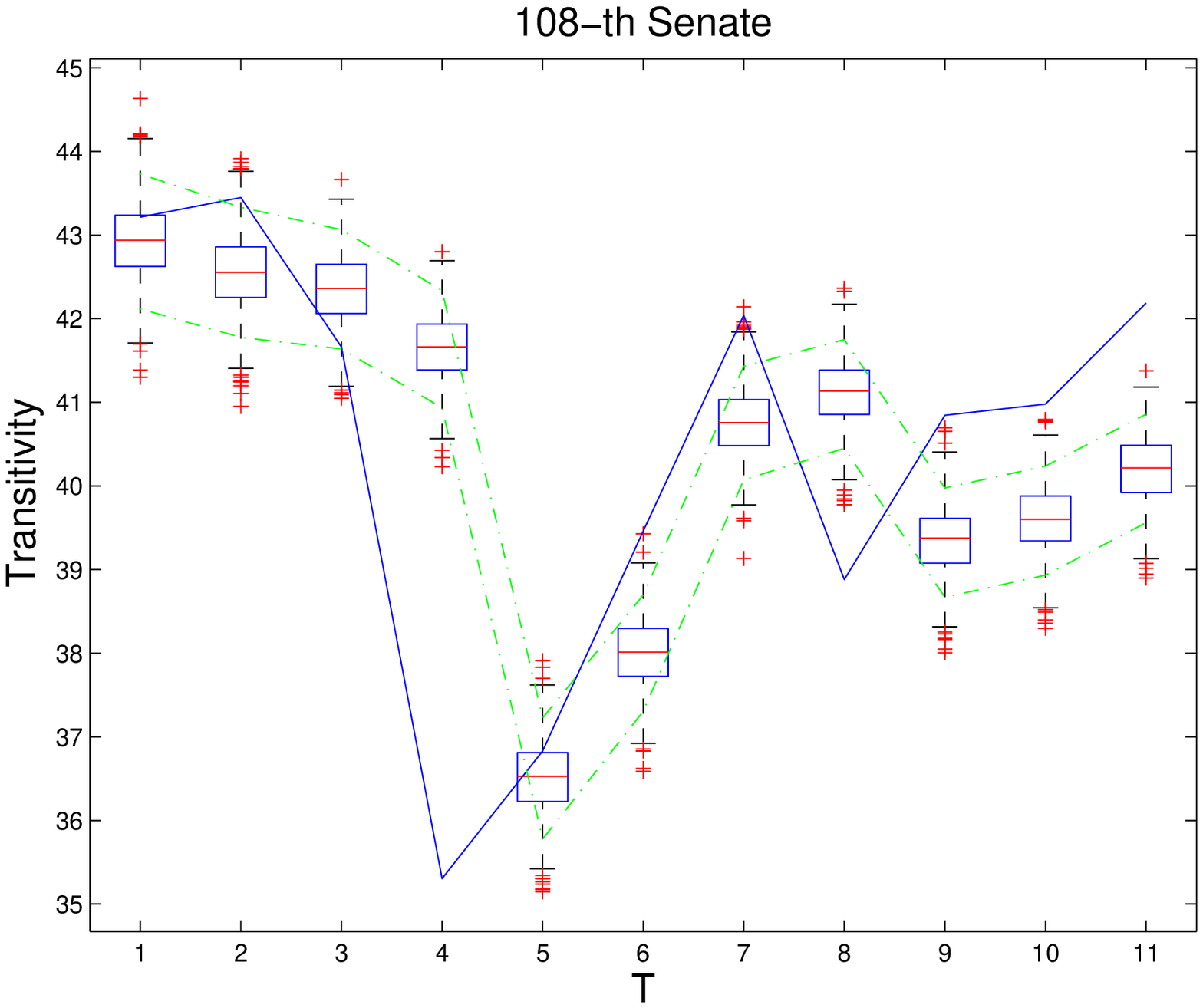}
\includegraphics[scale=\scA]{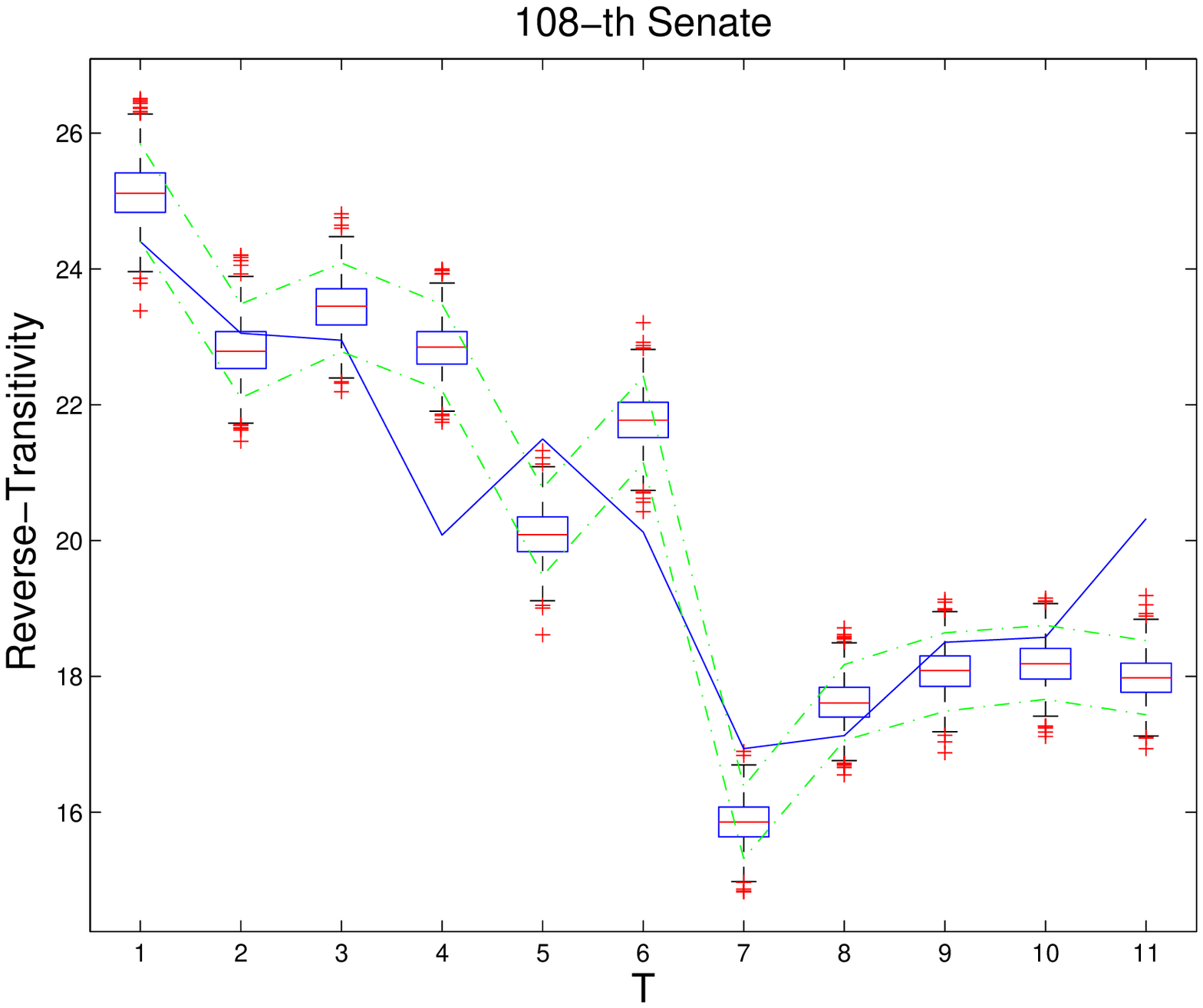}
\includegraphics[scale=\scA]{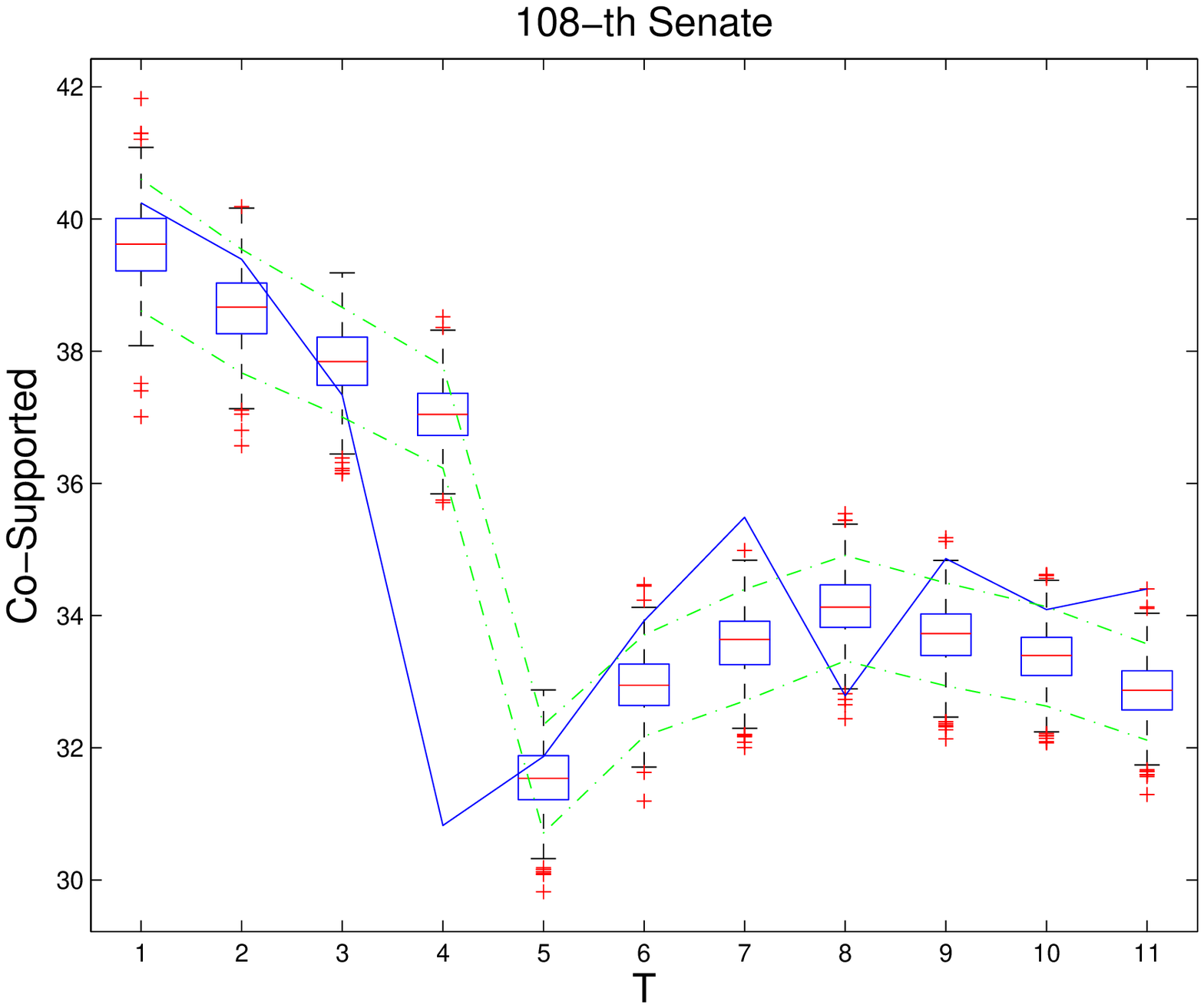}}
\centerline{\includegraphics[scale=\scA]{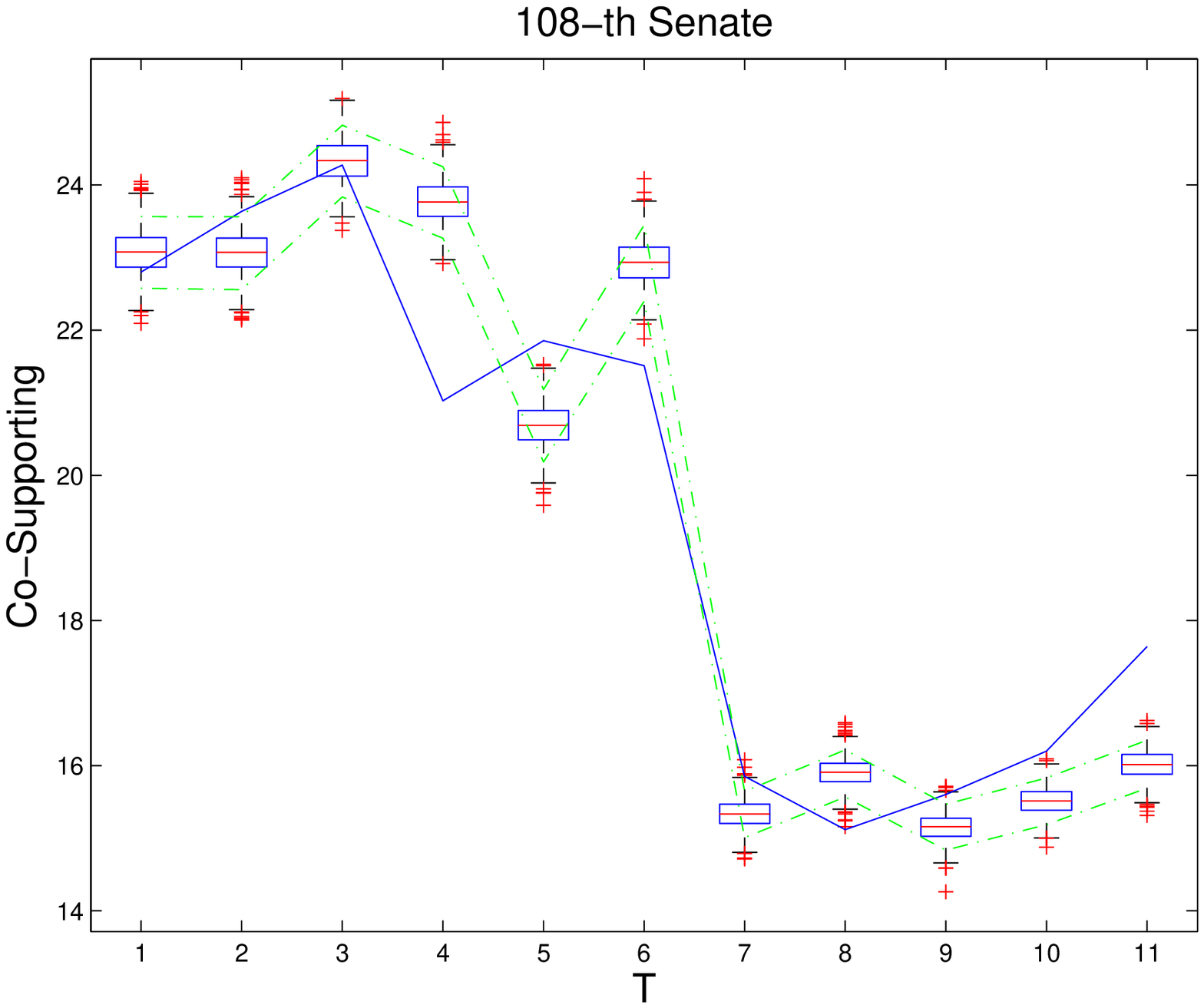}
\includegraphics[scale=\scA]{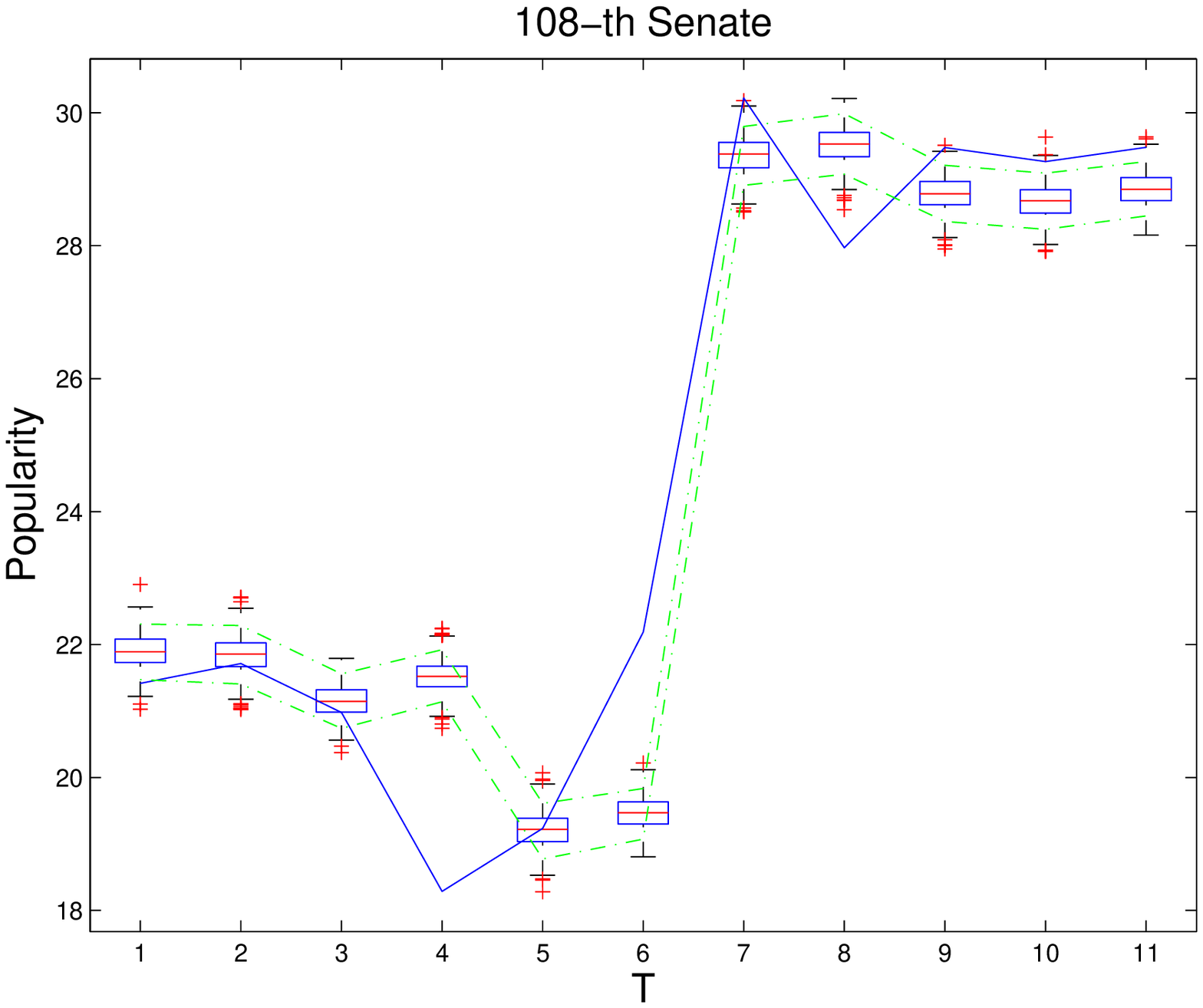}
\includegraphics[scale=\scA]{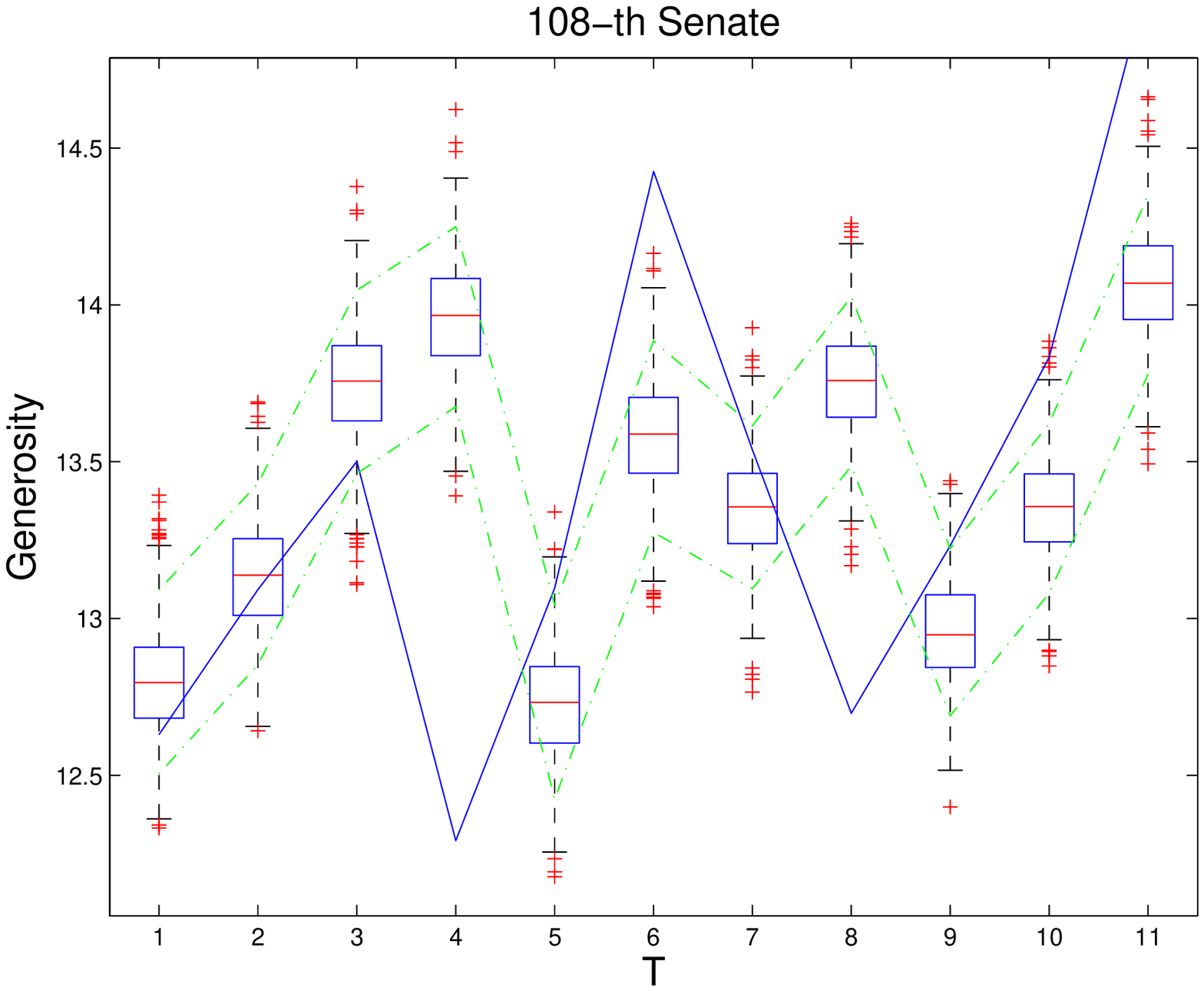}}
\caption{Statistic values of real networks and sampled networks based on a TERGM with 9 statistics. 
The comparisons are grouped by statistic. Blue solid lines indicate the observed (true) network statistics. 
Box plots are for the sampled networks (in the described cross-validation experiments) and green dotted lines indicate 5- and 95-percentiles.}
\label{fig:senate_n3}
\end{figure}

\section{Future Work}

If we think of this type of model as describing a process giving rise
to the networks one observes in reality, then one can think of a
single network observation as a snapshot of this Markov chain at that
time point.  Traditionally one would model a network at a single time
point using an ERGM.  However, in light of the degeneracy issues found
in ERGMs, and the lack thereof for the temporal models with product
conditional distributions, it seems worthwhile to investigate modeling
a single network as the end-point of an unobservable sequence.
Directly modeling this with latent variables would seem to make
inference computationally difficult.  However, it may be possible to
indirectly model this by studying the stationary distribution of these
Markov chains.  To our knowledge, it remains an open problem to
directly specify the family of stationary distributions that a given
TERGM corresponds to.

Moving forward, we hope to move beyond these ERG-inspired models
toward models that incorporate latent variables, which may also evolve
over time with the network.  For example, it may often be the case
that the phenomena represented in data can most easily be described by
imagining the existence of unobserved groups or factions, which form,
dissolve, merge and split as time progresses.  The flexibility of the
ERG models and the above temporal extensions allows a social scientist
to ``plug in'' his or her knowledge into the formulation of
the model, while still providing general-purpose estimation algorithms
to find the right trade-offs between competing and complementary
factors in the model.  We would like to retain this flexibility in
formulating a general family of models that include evolving latent variables
in the representation, so that the researcher can ``plug in'' his or her
hypotheses about latent group dynamics, evolution of unobservable
actor attributes, or a range of other possible phenomena into the
model representation.  At the same time, we would like to preserve the
ability to provide a ``black box'' inference algorithm to determine the
parameter and variable values of interest to the researcher, as can be
done with ERGMs and now TERGMs.

\subsection*{Acknowledgments}
We thank Stephen Fienberg and all participants in the statistical network modeling discussion group at Carnegie Mellon for helpful comments and feedback.

\bibliography{learning}                %% Generate the bibliography

\end{document}